\documentclass{article} 

\usepackage{colortbl,booktabs}
\definecolor{tabgrey}{rgb}{0.8,0.8,0.8}

\usepackage{amsmath}
\usepackage{amssymb}
\usepackage{amsthm}
\usepackage{natbib}

\newcommand{\D}{\nabla}
\newcommand{\R}{\mathbb{R}}

\newcommand{\regret}{\text{Regret}}

\newcommand{\inner}[1]{\left\langle#1\right\rangle}
\newcommand{\E}[2]{\mathbb{E}_{#1}\left[#2\right]}
\newcommand{\Lag}{\mathcal{L}}
\newcommand{\ones}{\mathbf{1}}

\newcommand{\ri}{\operatorname{ri}}
\newcommand{\dom}{\operatorname{dom}}

\newcommand{\DT}{\Delta_\Theta}
\newcommand{\DX}{\Delta_X}
\newcommand{\losst}{\lambda}
\newcommand{\lossx}{\ell}
\newcommand{\tlossx}{L}
\newcommand{\assess}{\alpha}

\DeclareMathOperator*{\argmin}{\arg\min}

\newcommand{\ie}{\textit{i.e.}}

\newtheorem{definition}{Definition}

\newtheorem{theorem}{Theorem}
\newtheorem{lemma}{Lemma}

\title{Generalised Mixability, Constant Regret, and Bayesian Updating}
\date{February 8th, 2014}
\author{%
	Mark D. Reid\\ 
	The Australian National University \& NICTA
\and
	Rafael M. Frongillo\\
	Microsoft Research
\and
	Robert C. Williamson\\
	The Australian National University \& NICTA
}

\begin{document}

\maketitle

\begin{abstract}
	Mixability of a loss is known to characterise when constant regret bounds
	are achievable in games of prediction with expert advice through the use of
	the aggregating algorithm~\citep{Vovk:2001}.
	We provide a new interpretation of mixability via convex analysis that
	highlights the role of the Kullback-Leibler divergence in its definition.
	This naturally generalises to what we call $\Phi$-mixability
	where the Bregman divergence $D_\Phi$ replaces the KL divergence.
	We prove that losses that are $\Phi$-mixable also enjoy constant regret
	bounds via a generalised aggregating algorithm that is similar to mirror
	descent.
\end{abstract}

\section{Introduction}

The combination or aggregation of predictions is central to machine
learning. 
Traditional Bayesian updating can be viewed as a particular way
of aggregating information that takes account of prior information.
Notions of ``mixability'' which play a central role in the setting of
prediction with expert advice offer a more general way to aggregate, and
which take account of the loss function used to evaluate predictions (how
well they fit the data). 
As shown by \citet{Vovk:2001}, his more general ``aggregating algorithm'' 
reduces to
Bayesian updating when log loss is used. 
However, as we will show there is another design variable that to date has
not been fully exploited.  
The aggregating algorithm makes use of a distance between the current
distribution and a prior which serves as a regulariser. 
In particular the aggregating algorithm uses the KL-divergence. 
We consider the general setting of an arbitrary loss and an arbitrary
regulariser (in the form of a Bregman divergence) and show that we recover
the core technical result of traditional mixability: if a loss is mixable in
the generalised sense then there is a generalised aggregating algorithm which
can be guaranteed to have constant regret.

In symbols (more formally defined later), if we use $\lossx_x(p_\theta)$ to
denote the loss of the prediction $p_\theta$ by expert $\theta$ on observation 
$x$ and $D_\phi(\mu',\mu)$ is used to penalise the
``distance'' between the choice of updated distribution $\mu'$ from what it
was previously $\mu$ then we can recover both Bayesian updating and the
updates of the aggregating algorithm as minimisers of
$\E{\theta\sim\mu'}{\lossx_x(p_\theta)} + D_\Phi(\mu',\mu)$ via the
choices summarised in the table below.
\begin{center}
	\arrayrulecolor{tabgrey}
	\renewcommand{\arraystretch}{1}
	\begin{tabular}{lcc}
		\toprule[0.3mm]
		Scheme & Loss  & Regulariser \\
		\midrule[0.3mm]
		Bayesian  updating & log loss & KL divergence \\
		Aggregating algorithm & general mixable loss & KL
		divergence \\
		This paper &  general $\Phi$-mixable loss & general Bregman
		divergence $D_\Phi$\\
		\bottomrule[0.3mm]
	\end{tabular}
\end{center}
We show that there is a single notion of mixability that applies to all three 
of these cases and guarantees the corresponding updates can be used to 
achieve constant regret.

We stress that the idea of the more general regularisation and updates is
hardly new. See for example the discussion of potential based methods in
\citep{Cesa-Bianchi:2006} and other references later in the paper. The key
novelty is the generalised notion of mixability, the name of which is
justified by the key new technical result --- a constant regret bound
assuming the general mixability condition achieved via a generalised
algorithm which can be seen as intimately related to mirror descent.
Crucially, our result depends on some properties of the conjugates of potentials 
defined over probabilities that do not hold for potential functions defined over 
more general spaces.

\subsection{Prediction With Expert Advice and Mixability} 

A prediction with expert advice game is defined by its loss, a collection 
of experts that the player must compete against, and a fixed number of rounds.
Each round the expert reveals their predictions to the player and then
the player makes a prediction.
An observation is then revealed to the experts and the player and all
receive a penalty determined by the loss.
The aim of the player is to keep its total loss close to that of the best
expert once all the rounds have completed.
The difference between the total loss of the player and the total loss of
the best expert is called the regret and is the typically the focus of
the analysis of this style of game.
In particluar, we are interested in when the regret is \emph{constant},
that is, independent of the number of rounds played.

More formally, let $X$ denote a set of possible \emph{observations}.
We consider a version of the game where predictions made by the player
and the experts are all distributions over $X$.
The set of such distributions will be denoted $\DX$ and the probability 
(or density) $p \in \DX$ assigns to $x \in X$ will be denoted $p(x)$.
A \emph{loss} $\lossx : \DX \to \R^X$ assigns the penalty $\ell_x(p)$ to 
predicting $p \in \Delta_X$ when $x \in X$ is observed.
The set of experts is denoted $\Theta$ and in each round 
$t = 1, \ldots, T$, each expert $\theta\in\Theta$ makes a prediction
$p^t_\theta \in \DX$.
These are revealed to the player who makes a prediction $p^t\in\DX$.
Once observation $x^t\in X$ is revealed the experts receive loss
$\lossx_{x^t}(p^t_\theta)$ and the player receives loss 
$\lossx_{x^t}(p^t)$.
The aim of the player is to minimise its \emph{regret}
\(
	\regret^T := \tlossx^T - \min_\theta \tlossx_\theta^T
\)
where $\tlossx^T := \sum_{t=1}^T \lossx_{x^t}(q^t)$ and 
$\tlossx_\theta^T = \sum_{t=1}^T \lossx_{x^t}(p^t_\theta)$.


The algorithm  that witnesses the original mixability result is known 
as the \emph{aggregating algorithm} (AA) \citep{Vovk:2001}.
It works similarly to exponentiated gradient algorithms
\citep{Cesa-Bianchi:2006} in that it updates a 
\emph{mixture distribution}\footnote{%
	To keep track of the two spaces $X$ and $\Theta$ we adopt the convention
	of using Roman letters for distributions in $\DX$ and vectors in 
	$\R^X$ and Greek letters for distributions in $\DT$ and vectors in 
	$\R^\Theta$.
}
$\mu \in \DT$ over experts based on their performance at the end of each round.
The mixture is then used to ``blend'' the predictions of the experts in the
next round in such a way as to achieve low regret.
In the aggregating algorithm, the mixture is intially set to some 
``prior'' $\mu^0 = \pi \in \DT$.
After $t-1$ rounds where the observations were $x^1, \ldots, x^{t-1}$ and the
expert predictions where $p^1_\theta, \ldots, p^{t-1}_\theta$ for 
$\theta\in\Theta$ the mixture is set to
\begin{equation}\label{eq:aa-updates}
	\mu^{t-1}(\theta) 
	=
	\frac{%
		\exp(-\eta\tlossx^{t-1}_\theta)
	}{%
		\sum_{\theta\in\Theta} \exp(-\eta\tlossx^{t-1}_\theta)
	}.
\end{equation}
On round $t$, after seeing all the expert predictions $p^t_\theta$, the 
AA plays a $p^t \in \DX$ such that for all $x\in X$
\begin{equation}\label{eq:vovk-mixable}
	\lossx_{x}(p^t) 
	\le
	-\frac{1}{\eta} 
	\log \sum_{\theta\in\Theta} \exp(-\eta \lossx_x(p^t_\theta)) \mu^t(\theta).
\end{equation}
Mixability is precisely the condition on the loss $\lossx$ that guarantees
that such a prediction $p^t$ can always be found.
\begin{definition}\label{def:vovk-mixable}
	A loss $\lossx : \DX \to \R$ is said to be $\eta$-\emph{mixable} 
	for $\eta > 0$ if for 
	all mixtures $\mu^t \in \DT$ and all predictions 
	$\{ p^t_\theta \}_{\theta\in\Theta}$ there exists a $p^t\in\DX$ such
	that \eqref{eq:vovk-mixable} holds for all $x \in X$.
\end{definition}

The key result concerning mixability is that it characterises when constant 
regret is achievable.
\begin{theorem}[\cite{Vovk:2001}]\label{thm:vovk}
	If $\lossx : \DX \to \R^X$ is $\eta$-mixable for some $\eta > 0$ then
	for any game of $T$ rounds with finitely many experts $\Theta$
	the aggregating algorithm will guarantee
	\[
		\sum_{t=1}^T \lossx_{x^t}(p^t) 
		\le 
		\sum_{t=1}^T \lossx_{x^t}(p^t_\theta) + \frac{\log |\Theta|}{\eta}.
	\]
\end{theorem}
Furthermore, \cite{Vovk:2001} also supplies a converse: that a constant 
regret bound is only achievable for $\eta$-mixable losses. 
Later work by \citet{Erven:2012} has show that mixability of proper
losses (see \S\ref{sub:preliminaries}) can be characterised in terms of the
curvature of the corresponding entropy for $\lossx$, that is, 
in terms of $\Phi^\lossx(p)=\langle p,\lossx(p)\rangle$.

\subsection{Contributions}
\label{sub:contributions}

Our main contribution is a generalisation of the notion of mixability and
a corresponding generalisation of Theorem~\ref{thm:vovk}.
Specifically, for any \emph{entropy} (\ie, convex function on the simplex)
$\Phi : \DT\to\R$ we define $\Phi$-mixability for losses $\lossx : \DX \to
\R^X$ (Definition~\ref{def:Phi-mixable}) and provide two equivalent
characterisations that lend themselves to some novel interpretations.
(Lemma~\ref{lem:mix-bound}).
We use these characterisations to prove the follow key result. 
Denote by
$\delta_\theta\in\Delta_\Theta$ the unit mass on $\theta$:
$\delta_\theta(\theta)=1$, $\delta_\theta(\theta')=0$ for all
$\theta'\ne\theta$.
	Let $D_\Phi$ denote the  Bregman divergence induced by $\Phi$,
	defined formally below in (\ref{eq:bregman-def}).
\begin{theorem}\label{thm:main}
	If $\lossx : \Delta_X \to \R^X$ is $\Phi$-mixable then there is
	a family of strategies parameterised by $\pi\in\DT$ which,
	for any sequence of observations 
	$x^1, \ldots, x^T \in X$ 
	and sequence of expert predictions $p^1_\theta, \ldots , p^T_\theta$, 
	plays a sequence $p^1, \ldots, p^T \in \Delta_X$ such that for all
	$\theta \in \Theta$
	\begin{equation}\label{eq:main}
		\sum_{t=1}^T \lossx_{x^t}(p^t)
		\le
		\sum_{t=1}^T \lossx_{x^t}(p^t_\theta) +
			D_\Phi(\delta_\theta, \pi) .
	\end{equation}
\end{theorem}

The standard notion of mixability is recovered when 
$\Phi = -\frac{1}{\eta}H$ for $\eta > 0$ and $H$ the Shannon entropy on $\DT$.
In this case, Theorem~\ref{thm:vovk} is obtained as a corollary
for $\pi$ the uniform distribution
over $\Theta$.
A compelling feature of our result is that it gives a natural interpretation
of the constant $D_\Phi(\delta_\theta, \pi)$ in the regret bound: if
$\pi$ is the initial guess as to which expert is best before the game starts,
the ``price'' that is paid by the player is exactly how far (as measured by
$D_\Phi$) the initial guess was from the distribution that places all its
mass on the best expert.

In addition, an algorithm analogous to the Aggregating Algorithm is naturally
recovered to witness the above bound during the construction of the proof;
see (\ref{eq:gen-aa-update}).
Like the usual Aggregating Algorithm, our ``generalised Aggregating Algorithm''
updates its mixtures according to the past performances of the
experts.
However, our algorithm is most easily understood as doing so via updates to 
the \emph{duals} of the distributions induced by $\Phi$.

\subsection{Related Work}

The starting point for mixability and the aggregating algorithm is the work
of \citet{Vovk:1995,Vovk:1990}.
The general setting of prediction with expert advice is summarised in
\cite[Chapters 2 and 3]{Cesa-Bianchi:2006}. There one can find a range of
 results that study diffferent aggregation schemes and different
 assumptions on the losses (exp-concave, mixable).
 Variants of the aggregating algorithm have been studied for classically
 mixable losses, with a tradeoff between tightness of the bound (in a
 constant factor) and the computational complexity \citep{Kivinen1999}.
Weakly mixable losses are a generalisation of mixable losses. They have
been studied in \cite{Kalnishkan2008} where it is  shown there exists a
variant of the aggregating algorithm that achives regret $C\sqrt{T}$ for
some constant $C$.
\citet[in \S2.2]{Vovk:2001} makes the observation that his Aggregating
Algorithm 
reduces to Bayesian mixtures in the case of the log loss game. See also the
discussion in \citep[page 330]{Cesa-Bianchi:2006} relating certain
aggregation schemes to Bayesian updating.

The general form of updating we propose is similar to that considered by 
\citet{Kivinen:1997} who consider finding a vector $w$ minimising
\(
	d(w,s) + \eta L(y_t, w\cdot x_t)
\)
where $s$ is some starting vector, $(x_t, y_t)$ is the instance/label
observation at round $t$ and $L$ is a loss.  The key difference between
their formulation and ours is that our loss term is (in their notation)
$w\cdot L(y_t, x_t)$ -- \ie, the linear combination of the losses of the
$x_t$ at $y_t$ and not the loss of their inner product.

Online methods of density estimation for exponential families are discussed in
\cite[\S3]{Azoury:2001} where they compare the online and offline updates of
the same sequence and make heavy use of the relationship between the KL
divergence between members of an exponential family and an associated Bregman 
divergence between the parameters of those members.

The analysis of mirror descent by \cite{Beck:2003} shows that it achieves 
constant regret when the entropic regulariser is used. 
However, they do not consider whether similar results extend to other entropies
defined on the simplex.


\section{Generalised Mixability} 

This work was motivated by the observation that the original mixability
definition \eqref{eq:vovk-mixable} looks very closely related to the 
log-sum-exp function, which is known to be the simplex-restricted conjugate 
of Shannon entropy.
We wondered whether the proof that mixability implies constant regret
was due to unique properties of Shannon entropy or whether alternative
notions of entropy could lead to similar results.
We found that the key step of the original mixability proof (that allows
the sum of bounds to telescope) holds for any convex function defined
on the simplex. 
This is because the conjugates of such functions have a translation
invariant property that allows the original telescoping series argument
to go through in the general case.
By re-expressing the original proof using only the tools of convex 
analysis we were able to naturally derive the corresponding update
algorithm and express the constant term in the bound as a Bregman
divergence.

\subsection{Preliminaries}
\label{sub:preliminaries}

We begin by introducing some basic concepts and notation from convex analysis.
Terms not defined here can be found in a reference such as 
\citep{Hiriart-Urruty:2001}.
A convex function $\Phi : \DT \to \R$ is called an \emph{entropy}
if it is proper, convex, and lower semi-continuous.
The \emph{Bregman divergence} associated with a suitably differentiable entropy 
$\Phi$ is given by
\begin{equation}
	\label{eq:bregman-def}
	D_\Phi(\mu, \mu')
	= 
	\Phi(\mu) - \Phi(\mu') - \inner{\D\Phi(\mu'), \mu - \mu'}
\end{equation}
for all $\mu \in \DT$ and $\mu' \in \ri(\DT)$, the relative interior of
$\DT$.
The \emph{convex conjugate} of $\Phi : \DT \to \R$ is defined to be
$\Phi^*(v) := 
\sup_{\mu\in\dom\Phi} \inner{\mu,v} - \Phi(\mu)=
\sup_{\mu\in\DT} \inner{\mu,v} - \Phi(\mu)$ where 
$v \in \DT^*$, \ie, the dual space to $\DT$.
One could also write the supremum over $\R^{\Theta}$ by the convention of
setting $\Phi(\mu) = +\infty$ for $\mu \notin\DT$.
For differentiable $\Phi$, it is known that the supremum defining $\Phi^*$ is 
attained at $\mu = \D\Phi^*(v)$ \citep{Hiriart-Urruty:2001}.
That is,
\begin{equation}\label{eq:conjugate-sup}
	\Phi^*(v) = \inner{\D\Phi^*(v),v} - \Phi(\D\Phi^*(v)).
\end{equation}
A similar result holds for $\Phi$ by applying this result to $\Phi^*$ and
using $\Phi = (\Phi^*)^*$.
We will make use of this result to establish the following inequality 
connecting a Bregman divergence $D_\Phi$ with its conjugate.
\begin{lemma}\label{lem:breg-div}
	For all $\mu\in\DT$ and $v \in \DT^*$ we have
	\[
		\Phi^*(\D\Phi(\mu)) - \Phi^*(\D\Phi(\mu) - v)
		=
		\inf_{\mu'\in\DT} \inner{\mu',v} + D_\Phi(\mu',\mu).
	\]
\end{lemma}
\begin{proof}
	By definition 
	$\Phi^*(\D\Phi(\mu) - v) 
	= \sup_{\mu'\in\DT} \inner{\mu',\D\Phi(\mu)-v} - \Phi(\mu')$ 
	and using \eqref{eq:conjugate-sup} expands $\Phi^*(\D\Phi(\mu))$
	to $\Phi^*(\D\Phi(\mu)) = \inner{\mu,\D\Phi(\mu) - \Phi(\mu)}$.
	Subtracting the former from the latter gives
	\[
		\inner{\mu,\D\Phi(\mu)} - \Phi(\mu)
		-\left[
			\sup_{\mu'\in\DT} \inner{\mu',\D\Phi(\mu)-v} - \Phi(\mu')
		\right]
	\]
	which, when rearranged gives
	\(
		\inf_{\mu'\in\DT}
			\Phi(\mu') - \Phi(\mu)
			- \inner{\D\Phi(\mu), \mu'-\mu}
			+ \inner{\mu',v}
	\)
	which then gives the result.
\end{proof}

We will also make use of a property of conjugates of entropies called 
\emph{translation invariance} \citep{Othman:2011}.
This notion is central to what are called convex and coherent risk functions
in mathematical finance \citep{Follmer:2004}.
In the following result and throughout, we use $\ones \in \R^\Theta$ for the 
point such that $\ones_\theta = 1$ for all $\theta\in\Theta$.

\begin{lemma}\label{lem:invariant}
	If $\Phi : \DT \to \R$ is an entropy then its convex conjugate is
	\emph{translation invariant}, that is, for all $v \in\DT^*$
	and $\alpha \in \R$ we have
	\(
		\Phi^*(v + \alpha\ones) = \Phi^*(v) + \alpha
	\)
	and its gradient satisfies 
	$\D\Phi^*(v + \alpha\ones) = \D\Phi^*(v)$.
\end{lemma}
\begin{proof}
	By definition of the convex conjugate we have
	\begin{align*}
		\Phi^*(v + \alpha\ones) 
		=& \sup_{\mu\in\DT} \inner{\mu,v + \alpha\ones} - \Phi(\mu)\\
		=& \sup_{\mu\in\DT} \inner{\mu,v} - \Phi(\mu) + \alpha\\
		=& \Phi^*(v) + \alpha
	\end{align*}
	since $\inner{\mu,\ones}=1$.
	Taking derivatives of both sides gives the second part of the lemma.
\end{proof}
We will also make use of the readily established fact that for any
convex $\Phi : \DT\to\R$ and all $\eta > 0$ we have 
$(\frac{1}{\eta} \Phi)^*(v) = \frac{1}{\eta}\Phi^*(\eta v)$.

Probably the most well studied example of what we call an entropy is the 
negative of the Shannon entropy\footnote{%
	We write Shannon entropy here as a sum but can also consider 
	the continuous version relative to some reference measure
	$\nu \in\DT$, that is, 
	$H(\mu) = - \int_{\DT} \log(\mu(\theta))\mu(\theta)\,d\nu(\theta)$.
	For simplicitly, we stick to the countable case.
}
$H(\mu) = -\sum_{\theta\in\Theta} \mu(\theta) \log \mu(\theta)$
which is known to be concave, proper, and upper semicontinuous and thus 
$\Phi = -H$ is an entropy.
When we look at the form of the original definition of mixability, we
observe that it is closely related to the conjugate of $(-H)$:  
\begin{equation}\label{eq:shannon-conjugate}
	(-H)^*(v) = \log \sum_{\theta\in\Theta} \exp(v(\theta))
\end{equation}
which is sometimes called the \emph{log-sum-exp} or 
\emph{partition function}.
This observation is what motivated this work and drives our generalisation
to other entropies.

Entropies are known to be closely related to the Bayes risk of what are
called proper losses or proper scoring rules~\citep{Dawid:2007,Gneiting:2007}.
Specifically, if a loss $\losst : \DT \to \R^\Theta$ is used to assign a
penalty $\losst_\theta(\mu)$ to a prediction $\mu$ upon outcome $\theta$
it is said to be \emph{proper} if its expected value under $\theta\sim\mu$
is minimsed by predicting $\mu$. 
That is, for all $\mu, \mu' \in \DT$ we have
\[
	\E{\theta\sim\mu}{\losst_\theta(\mu')}
	=
	\inner{\mu,\losst(\mu')} 
	\ge
	\inner{\mu,\losst(\mu)}
	=: -\Phi^\losst(\mu)
\]
where $-\Phi^\losst$ is the \emph{Bayes risk} of $\losst$ and is necessarily
concave~\citep{Erven:2012}, thus making $\Phi^\losst : \DT\to\R$ convex
and thus an entropy.
The correspondence also goes the other way: given any convex function 
$\Phi : \DT\to\R$ we can construct a unique proper loss.
The following representation can be traced back to ~\citet{Savage:1971} but
is expressed here using conjugacy.

\begin{lemma}\label{lem:proper-loss}
	If $\Phi:\DT\to\R$ is a differentiable entropy then the loss
	$\losst^\Phi :\DT\to\R$ defined by
	\begin{equation}\label{eq:proper-loss}
		\losst^\Phi(\mu)
		:=
		\Phi^*(\D\Phi(\mu))\ones - \D\Phi(\mu)
	\end{equation}
	is proper.
\end{lemma}
\begin{proof}
  By eq.~\eqref{eq:conjugate-sup} we have 
  $\Phi^*(\D\Phi(\mu)) = \inner{\mu,\D\Phi(\mu)} - \Phi(\mu)$, giving us
  \begin{align*}
    \inner{\mu,\losst^\Phi(\mu')} - \inner{\mu,\losst^\Phi(\mu)}
    &= \Bigl(\inner{\mu',\D\Phi(\mu')} - \Phi(\mu') 
	- \inner{\mu,\D\Phi(\mu')}\Bigr)\\ &\quad
    - \Bigl(\inner{\mu,\D\Phi(\mu)} - \Phi(\mu) - \inner{\mu,\D\Phi(\mu)}\Bigr)\\
    &= D_\Phi(\mu,\mu'),
  \end{align*}
  from which propriety follows.
\end{proof}

It is straight-forward to show that the proper loss associated with the
negative Shannon entropy $\Phi = -H$ is the log loss, that is,
$\losst^{-H}(\mu) := -\left(\log \mu(\theta)\right)_{\theta\in\Theta}$.

\subsection{$\Phi$-Mixability}
\label{sub:generalising_mixability}

For a loss $\lossx : \Delta_X \to \R^X$ define the \emph{assessment} 
$\assess : X \to \R^\Theta$ to be the loss of each model/expert $p_\theta$
on observation $x$, \ie,
\(
	\assess_\theta(x) := \lossx_x(p_\theta).
\)

\begin{definition}\label{def:Phi-mixable}
	Suppose $\Phi$ is a differentiable entropy on $\Delta_\Theta$.
	A loss $\lossx : \Delta_X \to \R^X$ is \emph{$\Phi$-mixable} 
	if for all $\{ p_\theta \}_\theta$ and all $\mu \in \DT$ there is a 
	$p\in \Delta_X$ such that for all $x \in X$,
	\begin{equation}\label{eq:mixable} 
		\lossx_x(p) \le
		-\Phi^*(-\losst^{\Phi}(\mu) - \assess(x)).
	\end{equation}
\end{definition}
We can readily show that this definition reduces to the standard mixability
definition when $\Phi = \frac{1}{\eta}(-H)$ since, in this case,
\begin{equation}\label{eq:eta-conjugate}
	\Phi^*(v) = \frac{1}{\eta}\log\sum_\theta \exp(\eta v(\theta))
\end{equation}
by \eqref{eq:shannon-conjugate} and the fact that 
$(\frac{1}{\eta}f)^*(x^*) = \frac{1}{\eta} f^*(\eta x^*)$ for any convex
$f$.
As mentioned above, the proper loss corresponding to this choice of $\Phi$ is 
easily seen to be 
$\losst_\theta^\Phi(\mu) = -\frac{1}{\eta} \log(\mu(\theta))$ by
substitution into \eqref{eq:proper-loss}.
Thus, the mixability inequality becomes
$\lossx_x(p) \le 
	-\frac{1}{\eta} \log \sum_\theta \exp(-\eta\assess(x) + \log \mu(\theta))$
which is equivalent to \eqref{eq:vovk-mixable}.

We now show that the above definition is equivalent to one involving the
Bregman divergence for $\Phi$ and also the difference in the ``potential''
$\Phi^*$ evaluated at $\D\Phi(\mu)$ before and after it is updated by
$\alpha(x)$.

\begin{lemma}\label{lem:mix-bound}
	Suppose $\Phi$ is a differentiable entropy on $\Delta_\Theta$.
	Then the $\Phi$-mixability condition~\eqref{eq:mixable} is equivalent to 
	the following:
	\begin{align}
      \label{eq:mix-divergence}
      \lossx_x(p) 
      &\le 
      \inf_{\mu'\in\DT} \inner{\mu',\assess(x)} + D_{\Phi}(\mu',\mu),
      \\
      \label{eq:mix-potential}
      \lossx_x(p)
      &\le
      \;\Phi^*(\D\Phi(\mu)) - \Phi^*(\D\Phi(\mu) - \assess(x)).
	\end{align}
\end{lemma}
\begin{proof}
	Expanding the definition of $\losst^\Phi(\mu)$ makes the right-hand
	side of~\eqref{eq:mixable} equal to
	\[
		-\Phi^*(-\Phi^*(\D\Phi(\mu))\ones + \D\Phi(\mu) - \assess(x))
		= -\Phi^*(\D\Phi(\mu) - \assess(x)) 
			+ \Phi^*(\D\Phi(\mu)) 
	\]
	since $\Phi^*$ is translation invariant by Lemma~\ref{lem:invariant}.  This
	gives~\eqref{eq:mix-potential}.  Further applying Lemma~\ref{lem:breg-div}
	with $v = \assess(x)$ gives~\eqref{eq:mix-divergence}.
\end{proof}

\section{The Generalised Aggregating Algorithm}
\label{sub:gaa}

In this section we prove our main result (Theorem~\ref{thm:main}) and examine 
the ``generalised Aggregating Algorithm'' that witnesses the bound.
The updating strategy we use is the one that 
repeatedly returns the minimiser of the right-hand side of 
\eqref{eq:mix-divergence}.
\begin{definition}\label{def:updates}
	On round $t$, after observing $x^t \in X$, the 
	\emph{generalised aggregating algorithm} (GAA) updates the mixture 
	$\mu^{t-1} \in \DT$ by setting
	\begin{equation}
		\label{eq:gen-aa-update}
		\mu^t 
		:= 
		\argmin_{\mu \in \DT}
			\inner{\mu,\assess(x^t)} + D_\Phi(\mu, \mu^{t-1}).
	\end{equation}
\end{definition}

The next lemma shows that this updating process simply aggregates the 
assessments in the dual space $\DT^*$ with $\D\Phi(\pi)$ as the starting
point.

\begin{lemma}\label{lem:updates}
	The GAA updates $\mu^t$ satisfy
	\(
		\D\Phi(\mu^t) 
		=
		\D\Phi(\mu^{t-1}) - \assess(x^t)
	\)
	for all $t$ and so
	\begin{equation}\label{eq:updates}
		\D\Phi(\mu^T) = \D\Phi(\pi) - \sum_{t=1}^T \alpha(x^t).
	\end{equation}
\end{lemma}
\begin{proof}
By considering the Lagrangian
\(
	\Lag(\mu,a)
	= 
	\inner{\mu,\assess(x^t)} + D_\Phi(\mu, \mu^{t-1}) + a(\inner{\mu,\ones}-1)
\)
and setting its derivative to zero we see that the minimising $\mu^t$ must 
satisfy
\(
	\D\Phi(\mu^t) 
	=
	\D\Phi(\mu^{t-1}) - \assess(x^t) - a^t\ones
\)
where the $a^t \in R$ is the dual variable at step $t$.
For convex $\Phi$, the functions $\D\Phi^*$ and $\D\Phi$ are inverses
\citep{Hiriart-Urruty:2001} so 
$\mu^t = \D\Phi^*(\D\Phi(\mu^{t-1}) - \assess(x^t) - a^t\ones)
       = \D\Phi^*(\D\Phi(\mu^{t-1}) - \assess(x^t))$
by the translation invariance of $\Phi^*$ (Lemma~\ref{lem:invariant}).
This means the constants $a^t$ are arbitrary and can be ignored.
Thus, the mixture updates satisfy the relation in the lemma and summing
over $t=1,\ldots,T$ gives \eqref{eq:updates}.
\end{proof}

To see how the updates just described are indeed a generalisation of those 
used by the original aggregating algorithm, we can substitute 
$\Phi = -\frac{1}{\eta} H$ and $\pi = \frac{1}{|\Theta|}$ 
in \eqref{eq:gen-aa-update}.
Because $H$ is maximal for uniform distributions we must have 
$\D\Phi(\pi) = -\frac{1}{\eta}\D H(\pi) = 0$ and so
$\mu^T = \D\Phi^*(-\sum_{t=1}^T \assess(x^t))$.
However, by \eqref{eq:mix-divergence} we see that 
\[
	\left[\D\Phi^*(v)\right]_\theta 
	= 
	\frac{e^{\eta v(\theta)}}{\sum_{\theta'} e^{\eta v(\theta')}}
\]
and then substituting $v(\theta) = -\sum_t \assess(x^t)$ gives the
update equation in \eqref{eq:aa-updates}.

\subsection{The proof of Theorem~\ref{thm:main}}\label{sub:proof}

Armed with the representations of $\Phi$-mixability in 
Lemma~\ref{lem:mix-bound} and the form of the updates in 
Lemma~\ref{lem:updates}, we now turn to the proof of 
our main result.

\begin{proof}[of Theorem~\ref{thm:main}]
By assumption, $\lossx$ is $\Phi$-mixable and so, for the updates $\mu^t$
just defined we have that there exists a $p^t \in \DX$ such that
\(
	\lossx_{x^t}(p^t)
	\le
	-\Phi^*(-\losst^\Phi(\mu^{t-1}) - \assess(x^t))
\)
for all $x^t \in X$.
Expressing these bounds using \eqref{eq:mix-potential} from 
Lemma~\ref{lem:mix-bound} and summing these over $t=1,\ldots,T$ gives
\begin{align}
	\sum_{t=1}^T \lossx_{x^t}(p^t)
	\le&
	\sum_{t=1}^T \Phi^*(\D\Phi(\mu^{t-1})) 
		- \Phi^*(\D\Phi(\mu^{t-1}) - \assess(x^t))
	\notag
	\\
	=&
	\Phi^*(\D\Phi(\mu^0)) - \Phi^*(\D\Phi(\mu^T))
	\label{eq:telescopes}\\
	=&
	\inf_{\mu'\in\DT} 
		\inner{\mu',\sum_{t=1}^T \assess(x^T)}
		+ D_\Phi(\mu',\pi)
	\label{eq:mu-T} \\
	\le&
	\inner{\mu',\sum_{t=1}^T \assess(x^t)} + D_\Phi(\mu',\pi)
	\qquad\text{ for all } \mu' \in \DT
	\label{eq:converse-needs-this}
\end{align}
Line \eqref{eq:telescopes} above is because 
$\D\Phi(\mu^t) = \D\Phi(\mu^{t-1}) - \assess(x^t)$ by Lemma~\ref{lem:updates} 
and the series telescopes.
Line ~\eqref{eq:mu-T} is obtained by applying \eqref{eq:gen-aa-update} from
Lemma~\ref{lem:updates} and Lemma~\ref{lem:breg-div}.
Setting $\mu' = \delta_\theta$ and noting 
$\inner{\delta_\theta,\assess(x^t)} = \lossx_{x^t}(p_\theta)$
gives the required result.
\end{proof}

Note that the proof above gives us something even stronger ---
eq.~\eqref{eq:converse-needs-this} states that the GAA satisfies the stronger
condition that eq.~\eqref{eq:main} hold for all $\mu\in\Delta_\Theta$, in
addition to all $\delta_\theta$, where the loss is an expected loss under
$\mu$.  In particular, choosing $T=1$ in eq.~\eqref{eq:mu-T}, we have
\begin{equation*}
  \lossx_{x^1}(p^1)
  \le
  \inf_{\mu\in\Delta_\Theta} \; \inner{\mu,\assess(x^1)} + D_\Phi(\mu,\pi),
\end{equation*}
from which we can conclude that $\lossx$ is actually $\Phi$-mixable from
Lemma~\ref{lem:mix-bound}.  Hence, an algorithm exists which guarantees the
bound in eq.~\eqref{eq:converse-needs-this} if and only if the loss $\lossx$ is
$\Phi$-mixable.

Finally, we briefly note some similarities between the Generalised Aggregating
Algorithm and the literature on automated market makers for prediction markets.
The now-standard framework of~\cite{abernethy2013efficient} defines the cost of
a purchase of some bundle of securities as the difference in a convex potential
function.  Formally, for some convex $C:\R^n\to\R$, a purchase of bundle
$r\in\R^n$ given current market state $q$ is given by $C(q+r) - C(q)$.  The
instantaneous prices in the market at state $q$ are therefore $p=\nabla C(q)$.
As the prices correspond to probabilities in their framework, it must be the
case that $R := C^*$ satisfies $\mathrm{dom}(R) = \Delta_n$.  From this we can
conclude as we have done above that $C$ is translation invariant, and thus one
can restate the cost of the bundle $r$ as $R^*(\nabla R(p)+r) - R^*(\nabla
R(p))$.

We now are in a position to draw an anology with our GAA.  The formulation of
$\Phi$-mixability in eq.~\eqref{eq:mix-potential} says that the loss upon
observing $x$ must be bounded above by $\Phi^*(\nabla \Phi(\mu))
- \Phi^*(\nabla \Phi(\mu)+\alpha(x))$, which is exactly the negative of the
expression above, where $R=\Phi$, $p=\mu$, and $r=\alpha(x)$.  Thus,
$\Phi$-mixability is saying the loss must be at least as good for the algorithm
than in the market making setting, and hence it is not surprising that the loss
bounds are the same in both settings; see~\cite{abernethy2013efficient} for
more details.

\section{Future Work} 
\label{sec:discussion_and_future_work}
\label{sub:future_work}

Our exploration into a generalized notion of mixability opens more doors than
it closes.  In the following, we briefly outline several open directions.

\subsection*{Relation to original mixability result}

The proof of our main result, Theorem~\ref{thm:main}, shows that in essence, an
algorithm can guarantee constant regret, expressed in terms of
a $\Phi$-divergence between a starting point and the best expert, if and only
if the underlying loss $\lossx$ is $\Phi$-mixable.  The original mixability
result of~\citet{Vovk:2001} states that one achieves a constant regret of $\log
|\Theta|/\eta$ if and only if $\lossx$ is, in our terminology,
$(-\eta^{-1}H)$-mixable.  But of course for any $\Phi$ which is bounded on
$\Delta_\Theta$, the penalty $D_\Phi(\delta_\theta,\pi)$ is also bounded, and
hence it would seem that for all bounded $\Phi:\Delta_\Theta\to\R$, a loss
$\lossx$ is mixable in the sense of Vovk if and only if it is
$\eta^{-1}\Phi$-mixable for some $\eta>0$.

\subsection*{Relation to curvatures of $\lossx$ and $\Phi$}

A recent result of~\citet{Erven:2012} shows that the mixability constant $\eta$
from the original Definition~\ref{def:vovk-mixable} can be calculated as the
ratio of curvatures between the Bayes risk of the loss $\lossx$ and Shannon
entropy.  It would stand to reason therefore that for any $\Phi$, the
$\Phi$-mixability constant $\eta$ for a loss $\lossx$, defined as the largest
$\eta$ such that $\lossx$ is $\eta^{-1}\Phi$-mixable, would be similarly
defined as the ratio to $-\Phi$ instead of $H$.

\subsection*{Optimal regret bound}

The curvature discussion above addresses the question of finding, given
a $\Phi$, the largest $\eta$ such that $\lossx$ is $\eta^{-1}\Phi$-mixable.
Note that the larger $\eta$ is, the smaller the corresponding regret term
$\eta^{-1}D_{\Phi}(\delta_\theta,\pi)$ is.  Hence, for fixed $\Phi$, this
$\Phi$-mixability constant yields the tightest bound.  The question remains,
however, what is the tightest bound one can acheive across \emph{all} choices
of $\Phi$?  Again in reference to Vovk, it seems that the choice of $\Phi$ may
not matter, at least as long as $D_\Phi(\delta_\theta,\pi)$ is a constant
independent of $\theta$.  It would be clarifying to directly assert this claim
or find a counter-example.



\bibliographystyle{plainnat}
\bibliography{genmix}

\end{document}